\documentclass[review,11pt]{ReportTemplate}
\usepackage{bm}
\usepackage{graphicx}
\usepackage{paralist,algorithmic,algorithm}
\usepackage{amsmath,mathrsfs,amssymb}
\newtheorem{lemma}{Proposition}

\def \x {\mathbf{x}}

\newenvironment{proof}{\textbf{Proof:}\ }{\hspace{\stretch{1}}$\square$\\}

\begin{document}
\begin{frontmatter}
\title{One-Pass Learning with Incremental and Decremental Features}

\author{Chenping Hou$^\dag~^\ddag$}
\author{Zhi-Hua Zhou$^\ddag$\corref{cor1}}
\address{ $^\dag$College of Science, National University of Defense Technology, Changsha, 410073 \\
$^\ddag$National Key Laboratory for Novel Software Technology, Nanjing University, Nanjing 210093, China} \cortext[cor1]{\small Corresponding author.
Email: zhouzh@nju.edu.cn}

\begin{abstract}
  In many real tasks the features are evolving, with some features being vanished and some other features augmented. For example, in environment monitoring some sensors expired whereas some new ones deployed; in mobile game recommendation some games dropped whereas some new ones added. Learning with such incremental and decremental features is crucial but rarely studied, particularly when the data coming like a stream and thus it is infeasible to keep the whole data for optimization. In this paper, we study this challenging problem and present the OPID approach. Our approach attempts to compress important information of vanished features into functions of survived features, and then expand to include the augmented features. It is the one-pass learning approach, which only needs to scan each instance once and does not need to store the whole data, and thus satisfy the evolving streaming data nature. The effectiveness of our approach is validated theoretically and empirically.
\end{abstract}


\end{frontmatter}

In many real applications, the features are evolving, with some features being vanished and some other features augmented. For example, in the research of environment monitoring, in order to detect the environment in full aspects, different kinds of sensors, such as trace metal, radioisotope, volatile organic compound and biological sensors \cite{s5010004}, are deployed in a dynamic way. Due to the differences in working ways and working conditions, such as optical, electrochemical or gravimetric \cite{Stetter}, some sensors expired whereas some new sensors deployed. If we regard the output of each sensor as a feature, the data features are both incremental and decremental. This setting also occurs in the mobile game recommendation system in Android market \cite{Skocir2012}. There are a lot of people making ratings for many games. Some games dropped whereas some new ones added with times elapsing. It is also a feature evolution system.

Compared with traditional problems, there are at least two challenges in analyzing this kinds of data. (1) In these applications, the instances are coming like a stream. It is different from traditional learning paradigm, since the features and instances are evolving simultaneously. (2) Due to the streaming nature, it is infeasible to keep the whole data. It requires us to access the data in one-pass way.

Learning with such incremental and decremental features and evolving instances is crucial but rarely studied. Several related works have been proposed to solve part of this problem. To manipulate instance evolving problem, online learning, which can be traced back to Perceptron algorithm \cite{Rosenblatt58theperceptron}, is a standard learning paradigm and there are plenty of researches recently \cite{nips/CauwenberghsP00, Cesa-Bianchi:2006:PLG:1137817, Hazan:2007:LRA:1296038.1296051, Zhang:2015:OBL}. One-pass learning, as a special case of online learning, has also attracted many research interests in recent years \cite{GaoJZZ13, OPMVL}. To solve feature incremental and decremental problem, there are some researches concerning missing and corrupted features, such as \cite{Globerson, icml/DekelS08, nips/TeoGRS07,icml/HazanLM15}.

Although these researches have achieved prominent performances in their learning settings, they cannot fulfill the requirements arisen from the above-mentioned real applications, since they only focus on \emph{one aspect} of instance and feature evolution problem, either instance evolution or feature evolution. Direct combination of two types of previous methods cannot deal with the problem well.

In this paper, we try to manipulate this problem by proposing One-Pass Incremental and Decremental learning approach (OPID). For clarity, we tackle it in two stages, i.e., Compressing stage (C-stage) and Expanding stage (E-stage). In C-stage, we propose a new one-pass learning method by compressing important information of vanished features into functions of survived features. It only accesses the instance once and extract useful information to assist the following learning task. In E-Stage, accompanied by the learning model trained in C-stage, we present a new learning method which can include augmented features and inherit the advantages from C-stage. As far as we know, this is the first research about the problem with simultaneous instance and feature evolution. Several theoretical and experimental results are provided for illustration.

\section{Preliminaries}

In our work, we mainly focus on the learning problem in two stages, i.e., C-stage and E-stage with feature and instance evolution simultaneously. In C-stage, data are collected in mini-batch style. Assume that there are totally $T_1$ batches. In each batch, the features of each instance can be divided into two parts. The first part contains the features which will vanish in E-stage and the other part consists of the features which survive for both stages. They are named as \emph{vanished feature} and \emph{survived feature}. In E-stage, we assume that there are two batches of data. One batch is employed for training and the other one is used for testing. The features of each instance in this stage can also be divided into two parts. The first part contains the features that are survived in both stages. The second part is augmented features, which is referred as \emph{augmented feature} in the following.

\begin{figure}[!t]
\centering
\label{fig:1}
\includegraphics[width=5.in]{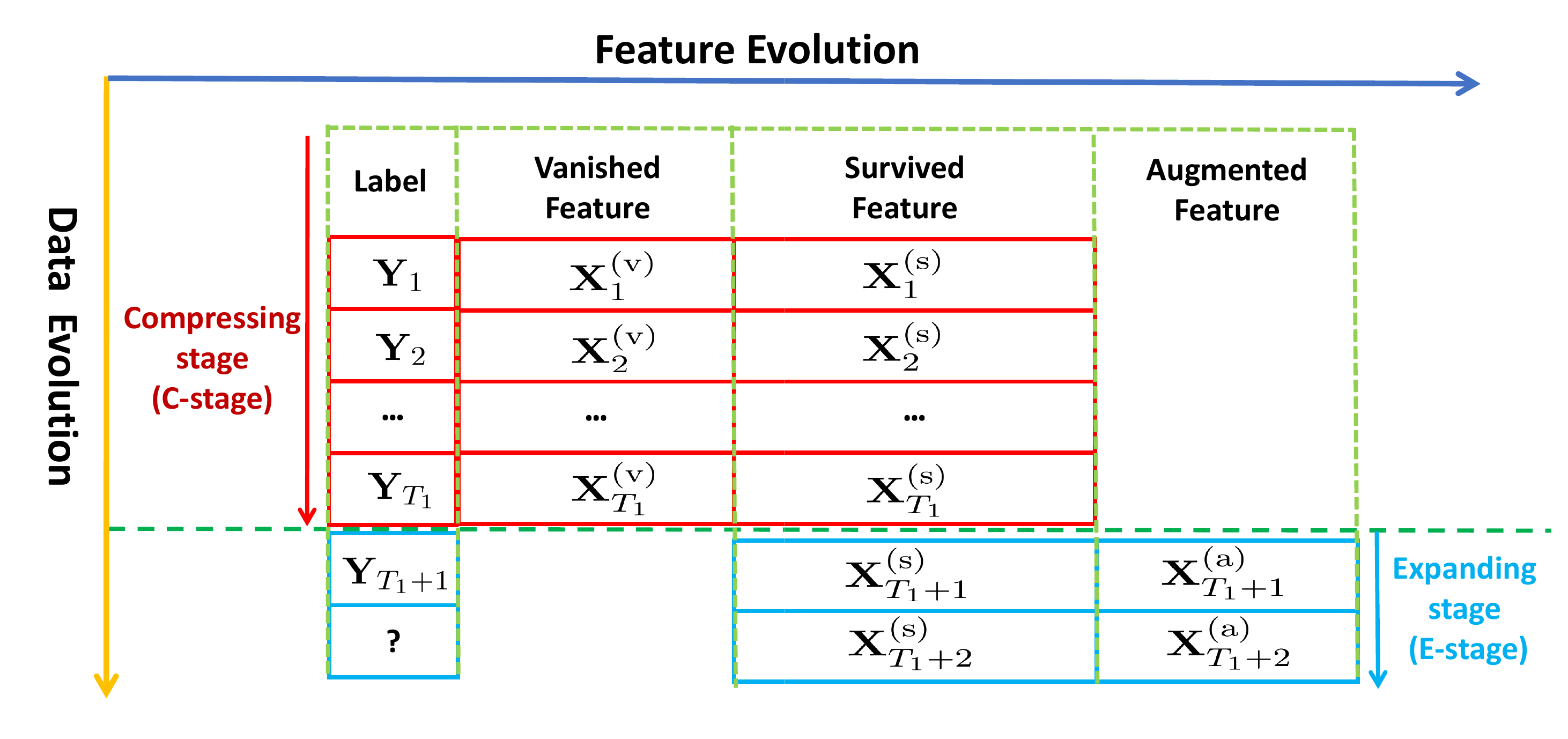}
\caption{Notations.}
\vskip -0.2in
\end{figure}

Formally, as pictured in Fig. 1, in the $i$-batch of C-stage, data points can be represented by two matrices, i.e., $\mathbf{X}_i^{\textrm{(v)}} \in \mathbb{R}^{n_i\times d^{\textrm{(v)}}}$ and $\mathbf{X}_i^{\textrm{(s)}} \in \mathbb{R}^{n_i\times d^{\textrm{(s)}}}$, where $n_i$ is the number of points in this batch, $d^{\textrm{(v)}}$ and $d^{\textrm{(s)}}$ are the numbers of vanished features and survived features respectively. Here, the superscripts "$\textrm{(v)}$" and "$\textrm{(s)}$" correspond to vanished features and survived features. The $j$-th row of $\mathbf{X}_i^{\textrm{(v)}}$ is an instance with only vanished features. Correspondingly, the $j$-th row of $\mathbf{X}_i^{\textrm{(s)}}$ is an instance with only survived features. The label matrix of instances in the $i$-batch is denoted by $\mathbf{Y}_i \in \mathbb{R}^{n_i \times c}$ with $c$ as the number of class. Its $(k,l)$-element $Y_i(k,l)=1$ if and only if the $k$-th instance in the $i$-batch belongs to the $l$-th category and $Y_i(k,l)=0$ otherwise.

In E-stage, it is often that we want to make prediction when we only get one batch training data. Thus, there are only two batches in this stage. The first batch contains training data, represented by $\mathbf{X}_{T_1+1}^{\textrm{(s)}} \in \mathbb{R}^{n_{T_1+1} \times d^{\textrm{(s)}}}$ and $\mathbf{X}_{T_1+1}^{\textrm{(a)}} \in \mathbb{R}^{n_{T_1+1} \times d^{\textrm{(a)}}}$, where $n_{T_1+1}$ is the number of training points in this stage, $d^{\textrm{(a)}}$ is the numbers of augmented features and the superscript "$\textrm{(a)}$" corresponds to augmented features. Similarly, the $j$-th row of $\mathbf{X}_{T_1+1}^{\textrm{(s)}}$ consists of survived features and the $j$-th row of $\mathbf{X}_{T_1+1}^{\textrm{(a)}}$ contains augmented features. The label matrix is denoted by $\mathbf{Y}_{T_1+1}$. Similarly, the second batch contains testing points with the same structure as training.

According to above notations, \textbf{our main task} is to classify data points represented by $\bar{\mathbf{X}}_{T_1+2} \triangleq [\mathbf{X}_{T_1+2}^{\textrm{(s)}}, \mathbf{X}_{T_1+2}^{\textrm{(a)}}]$, based on $\{ \bar{\mathbf{X}}_{T_1+1} \triangleq [\mathbf{X}_{T_1+1}^{\textrm{(s)}}, \mathbf{X}_{T_1+1}^{\textrm{(a)}}],~\mathbf{Y}_{T_1+1}\}$. Besides, we can also use models learned on $\{\tilde{\mathbf{X}}_i \triangleq [\mathbf{X}_i^{\textrm{(v)}}, \mathbf{X}_i^{\textrm{(s)}}],~\mathbf{Y}_{i}\}_{i=1}^{T_1}$, without saving all the data. In the $T$-th step, we can only access the data $\{\tilde{\mathbf{X}}_T, ~\mathbf{Y}_T\}$, without saving the past data $\{\tilde{\mathbf{X}}_i, ~\mathbf{Y}_{i}\}_{i=1}^{T-1}$. Here, \emph{a tilde above the symbol represents variable in C-stage and a bar represents variable in E-stage}.

It is noteworthy to mention that we can also use $\{\bar{\mathbf{X}}_{T_1+1}, \mathbf{Y}_{T_1+1}\}$ merely to train a classifier in E-stage. Nevertheless, in many real applications, since $n_{T_1+1}$ is often comparable to that of $n_{i}$ for $i=1,2,\cdots,n_{T_1}$ and they are often small due to the limitation of storage, training with only the instance in E-stage trends to be over-fitting. It is better to learn a classifier with the assistance of model trained in C-stage.

\section{The OPID Approach}

We investigate a real application problem with complicated settings, and it is difficult to use traditional approaches to solve this problem directly. There are two stages and both the instances and features are changed. In our paper, we tackle this problem in following way.

~~1. In C-stage, we learn a classifier based on $\{\tilde{\mathbf{X}}_i, \mathbf{Y}_{i}\}_{i=1}^{T_1}$ in one-pass way. That is, we only access training examples once. Besides, the learned classifier should also provide useful classification information to help the training in E-stage.

~~2. In E-stage, we learn a classifier based on $\{\bar{\mathbf{X}}_{T_1+1}, \mathbf{Y}_{T_1+1}\}$, under the supervision of classifier learned in C-stage.

\textbf{C-stage}: In this stage, there are two different kinds of features, i.e., vanished feature and survived feature. If we combine them to learn a unified classifier, it will be difficult to use it in E-stage, since the features are different in two stages. Notice that, there are features surviving in both stages. It is better to compress important information of vanished features into functions of survival features. In other words, we want to use the model trained in survived features to represent important information of both vanished features and survived features.

Let $\tilde{\mathcal{H}}$ and $\mathcal{H}^{\textrm{(s)}}$ be the function space for all features (vanished features and survived features) and the survived features respectively. In C-stage, we try to learn two classifiers $ \tilde{h} \in \tilde{\mathcal{H}}$ and $ \tilde{h}^{\textrm{(s)}} \in \mathcal{H}^{\textrm{(s)}}$ with some consistency constraints between them. Denote the loss function on all features and survived features as $\tilde{\ell}$ and $\tilde{\ell}^{\textrm{(s)}}$, the expected classifiers in C-stage can be learned by optimizing
{
\begin{equation}
\label{eq1}
\begin{split}
\min_{\tilde{h} \in \mathcal{H}, \tilde{h}^{\textrm{(s)}} \in \mathcal{H}^{\textrm{(s)}} }~~& \sum_{i=1}^{T_1}
 \tilde{\ell} \left( \tilde{h} (\tilde{\mathbf{X}}_i), \mathbf{Y}_i \right)+\tilde{\ell}^{\textrm{(s)}}\left( \tilde{h}^{\textrm{(s)}}(\mathbf{X}_i^{\textrm{(s)}}), \mathbf{Y}_i \right) \\
\textrm{s.t.}~~ & \mathcal{D}\left( \tilde{h}(\tilde{\mathbf{X}}_i) , \tilde{h}^{\textrm{(s)}}(\mathbf{X}_i^{\textrm{(s)}}) \right) \leq \epsilon. \textrm{~~for~~} i\in[T_1]~.\\
\end{split}
\end{equation}}
Here $\tilde{\mathbf{X}}_i = [\mathbf{X}_i^{\textrm{(v)}}, \mathbf{X}_i^{\textrm{(s)}}]$. $\mathcal{D}$ is employed to measure the consistency between two classifiers on every batch. We denote $[T_1]=\{1,2,\cdots,T_1\}$ for the convenience of presentation.

For example, we assume $\tilde{h}$ and $\tilde{h}^{\textrm{(s)}}$ are two linear classifiers with square loss. Besides, the Frobenius norm (denoted by $\|\cdot\|$) is employed as the measurement of consistency. The optimization problem in Eq. (\ref{eq1}) becomes
{
\begin{equation}
\label{eq2}
\begin{split}
\min_{\tilde{\mathbf{W}}, \mathbf{W}^{\textrm{(s)}}}~~& \sum_{i=1}^{T_1} \| \langle \tilde{\mathbf{W}}, \tilde{\mathbf{X}}_{i} \rangle - \mathbf{Y}_{i} \|^2 + \sum_{i=1}^{T_1} \| \langle \mathbf{W}^{\textrm{(s)}}, \mathbf{X}_{i}^{\textrm{(s)}} \rangle- \mathbf{Y}_{i} \|^2 \\
&+\lambda \sum_{i=1}^{T_1} \| \langle \tilde{\mathbf{W}}, \tilde{\mathbf{X}}_{i} \rangle - \langle \mathbf{W}^{\textrm{(s)}}, \mathbf{X}_{i}^{\textrm{(s)}} \rangle \|^2+
\rho(\| \tilde{\mathbf{W}} \|^2 + \| \mathbf{W}^{\textrm{(s)}}\|^2),
\end{split}
\end{equation}}
where $\tilde{\mathbf{W}}$ and $\mathbf{W}^{\textrm{(s)}}$ are classifier coefficients defined on all features and survived features in C-stage respectively, $\lambda>0$ is the parameter to tune the importance of consistency constraint and $\rho>0$ is the parameter for regularization. It is an extension of traditional regularized least square classifier by adding the consistency constraint. We will solve it by two types of one-pass learning methods.

\textbf{E-stage}: We expand the classifier trained on survived features, i.e., $\tilde{h}^{\textrm{(s)}}$, in C-stage to accommodate augmented features. It can take $\mathbf{X}_{T_1+1}^{\textrm{(s)}}$ as the input directly. Denote $\tilde{h}_{*}^{\textrm{(s)}}$ as the optimal classifier and $\mathbf{W}_{*}^{\textrm{(s)}}$ as its coefficient in C-stage, then $\mathbf{Z}_{T_1+1}^{\textrm{(s)}} \triangleq \tilde{h}_{*}^{\textrm{(s)}}(\mathbf{X}_{T_1+1}^{\textrm{(s)}}) = \mathbf{X}_{T_1+1}^{\textrm{(s)}} \mathbf{W}_{*}^{\textrm{(s)}}$ is the prediction by employing the classifier training in C-stage. We take this prediction as new representations of $\mathbf{X}_{T_1+1}^{\textrm{(s)}}$ as in stacking \cite{stacking, Zhou:2012:EMF}. After that, we train a classifier $\bar{h}^{\textrm{(s)}}$ on $\mathbf{Z}_{T_1+1}^{\textrm{(s)}}$ and simultaneously, another classifier $\bar{h}$ is trained on $ \bar{\mathbf{Z}}_{T_1+1} \triangleq [\mathbf{Z}_{T_1+1}^{\textrm{(s)}}, \mathbf{X}_{T_1+1}^{\textrm{(a)}}]$ for accommodation. It can be regarded as expansion of the optimal classifier in C-stage to include augmented features.

To inherit advantages of $\tilde{h}_{*}^{\textrm{(s)}}$ trained in C-stage, we combine two classifiers like ensemble methods \cite{Zhou:2012:EMF}. At first, we employ this strategy to unify two classifiers by optimizing the following problem.
{
\begin{equation}
\label{eq11}
\begin{split}
\min_{\bar{h}^{\textrm{(s)}}, \bar{h}} w_1~\bar{\ell}^{\textrm{(s)}} \left( \bar{h}^{\textrm{(s)}}(\mathbf{Z}_{T_1+1}^{\textrm{(s)}}),\mathbf{Y}_{T_1+1} \right) + w_2~\bar{\ell} \left( \bar{h} (\bar{\mathbf{Z}}_{T_1+1}),\mathbf{Y}_{T_1+1} \right)~,
\end{split}
\end{equation}}
where $w_1 \geq 0, w_2\geq 0, w_1+w_2=1$, are the weights to balance two classifiers. $\bar{\ell}^{\textrm{(s)}}$ and $\bar{h}$ are two surrogate loss function defined on $\mathbf{Z}_{T_1+1}^{\textrm{(s)}}$ and $\bar{\mathbf{Z}}_{T_1+1}$ respectively.

For simplicity, we use the L$_2$-regularized Logistic Regression model for each classifier. Take the binary classification problem as an example, the objective functions are
{
\begin{equation}
\label{eq19}
\begin{split}
&\bar{\ell}^{\textrm{(s)}} \left( \bar{h}^{\textrm{(s)}}(\mathbf{Z}_{T_1+1}^{\textrm{(s)}}),\mathbf{y}_{T_1+1} \right) = \frac{1}{2} \mathbf{v}^{\textrm{(s)}}(\mathbf{v}^{\textrm{(s)}})^\top+ \alpha_1 \sum_{j=1}^{n_{T_1+1}} \textrm{log}(1+\exp (- y_{T_1+1,j}\mathbf{z}_{T_1+1,j}^{\textrm{(s)}}\mathbf{v}^{\textrm{(s)}} )) \\
&\bar{\ell} \left( \bar{h} (\bar{\mathbf{Z}}_{T_1+1}),\mathbf{y}_{T_1+1} \right) = \frac{1}{2} \bar{\mathbf{v}} \bar{\mathbf{v}}^\top+ \alpha_2 \sum_{j=1}^{n_{T_1+1}} \textrm{log}(1+\exp (- y_{T_1+1,j} \bar{\mathbf{z}}_{T_1+1,j} \bar{\mathbf{v}} ))
\end{split}
\end{equation}}
where $\mathbf{v}^{\textrm{(s)}}$ and $\bar{\mathbf{v}}$ are coefficients. $\mathbf{z}_{T_1+1,j}^{\textrm{(s)}}$ and $\bar{\mathbf{z}}_{T_1+1,j}$ are the $j$-th row (the $j$-th instances) of $\mathbf{Z}_{T_1+1}^{\textrm{(s)}}$ and $\bar{\mathbf{Z}}_{T_1+1}$. $\alpha_1$ and $\alpha_2$ are balance parameters. $y_{T_1+1,j}$ is 1 or -1 for binary classification.

One point should be mentioned here. We take $\mathbf{Z}_{T_1+1}^{\textrm{(s)}}$ as the new representation, although it can be regarded as the classification results directly. The reasons are: (1) The prediction $\mathbf{Z}_{T_1+1}^{\textrm{(s)}}$ is computed by the classifier trained on C-stage, we can use training data in E-stage to improve $\tilde{h}_{*}^{\textrm{(s)}}$. This composite works since $\tilde{h}_{*}^{\textrm{(s)}}$ and $\bar{h}^{\textrm{(s)}}$ are trained on different data sets. (2) As $\sum_{i=1}^{T_1} n_{i}$ is often much larger than $n_{T_1+1}$, by contrast with the classifier trained on $\mathbf{X}_{T_1+1}^{\textrm{(s)}}$ merely, $\tilde{h}_{*}^{\textrm{(s)}}$ is a better classifier and it could extract more discriminative information. In other words, compared with $\mathbf{X}_{T_1+1}^{\textrm{(s)}}$, $\mathbf{Z}_{T_1+1}^{\textrm{(s)}}$ is a more compact and accurate representation. (3) If $\tilde{h}_{*}^{\textrm{(s)}}$ is good enough, the composite of $\bar{h}^{\textrm{(s)}}$ will not degrade the performance by our following strategy.

\section{Optimization and Extension}

\subsection{Optimization}

\textbf{C-stage:} The optimization problem in Eq.(\ref{eq1}) can be divided into $T_1$ subproblems, thus, it is direct to use the online learning method, such as online ADMM \cite{icml/WangB12}, to solve it by scanning the data only once. Nevertheless, we aim to train $\tilde{h}^{\textrm{(s)}}$ to assist the classification in E-stage, and the most direct way is to employ a linear classifier. In this case, we will provide more effective one-pass learning methods than using ADMM.

\begin{lemma}
\label{lemma1}
The optimal solution to Eq. (\ref{eq2}) can be obtained by solving
{
\begin{equation}
\label{eq4}
\begin{split}
\mathbf{A}_{[T_1]} \left[ {\begin{array}{*{20}{c}}
\tilde{\mathbf{W}} \\
\mathbf{W}^{\textrm{(s)}}
\end{array}} \right]=\mathbf{B}_{[T_1]},
\textrm{~~with~~}
\mathbf{B}_{[T]} \triangleq \sum_{i=1}^T
\left[ {\begin{array}{*{20}{c}}
\tilde{\mathbf{X}}_i^\top \\
(\mathbf{X}_i^{\textrm{(s)}})^\top \\
\end{array}} \right]\mathbf{Y}_i~,
\end{split}
\end{equation}}
{
\begin{equation}
\label{eq5}
\begin{split}
\mathbf{A}_{[T]} \triangleq
\left[ {\begin{array}{*{20}{c}}
(1+\lambda) \sum_{i=1}^{T} \tilde{\mathbf{X}}_{i}^\top \tilde{\mathbf{X}}_{i} +\rho \mathbf{I} &  -\lambda \sum_{i=1}^{T} \tilde{\mathbf{X}}_{i}^\top \mathbf{X}_{i}^{\textrm{(s)}} \\
-\lambda \sum_{i=1}^{T} (\mathbf{X}_{i}^{\textrm{(s)}})^\top \tilde{\mathbf{X}}_{i} & (1+\lambda) \sum_{i=1}^{T} (\mathbf{X}_{i}^{\textrm{(s)}})^\top \mathbf{X}_{i}^{\textrm{(s)}} +\rho \mathbf{I}
\end{array}}
\right],
\end{split}
\end{equation}}
and $\mathbf{I}$ is an identity matrix.
\end{lemma}

\begin{proof}
Take the derivative of the objective function in Eq. (\ref{eq2}) with respect to $\tilde{\mathbf{W}}$ and $\mathbf{W}^{\textrm{(s)}}$ and set them to zeros, we have the following equations.
{
\begin{equation}
\label{eq3-s}
\begin{split}
&\sum_{i=1}^{T_1} \tilde{\mathbf{X}}_{i}^\top (\tilde{\mathbf{X}}_{i} \tilde{\mathbf{W}} - \mathbf{Y}_{i}) + \lambda \sum_{i=1}^{T_1} \tilde{\mathbf{X}}_{i}^\top (\tilde{\mathbf{X}}_{i} \tilde{\mathbf{W}} - \mathbf{X}_{i}^{\textrm{(s)}}\mathbf{W}^{\textrm{(s)}}) +\rho \tilde{\mathbf{W}} = \mathbf{0},\\
&\sum_{i=1}^{T_1} (\mathbf{X}_{i}^{\textrm{(s)}})^\top (\mathbf{X}_{i}^{\textrm{(s)}}\mathbf{W}^{\textrm{(s)}}- \mathbf{Y}_{i})+\lambda \sum_{i=1}^{T_1} (\mathbf{X}_{i}^{\textrm{(s)}})^\top ( \mathbf{X}_{i}^{\textrm{(s)}} \mathbf{W}^{\textrm{(s)}}- \tilde{\mathbf{X}}_{i} \tilde{\mathbf{W}}) + \rho \mathbf{W}^{\textrm{(s)}}=\mathbf{0}.\\
\end{split}
\end{equation}
}
Denote $\mathbf{A}_{[T]}$ and $\mathbf{B}_{[T]}$ as shown in Eq. (\ref{eq5}) and Eq. (\ref{eq4}). The optimization problem problem in Eq. (\ref{eq3-s}) becomes
{
\begin{equation}
\label{eq6-s}
\begin{split}
\mathbf{A}_{[T_1]} \left[ {\begin{array}{*{20}{c}}
\tilde{\mathbf{W}} \\
\mathbf{W}^{\textrm{(s)}}
\end{array}} \right]=\mathbf{B}_{[T_1]}.
\end{split}
\end{equation}
}
It is just the results shown in Proposition 1.
\end{proof}

 Based on the above deduction, we turn to solve the problem in Eq. (\ref{eq2}) in one-pass way quickly. In the $T$-th time, we only access the instance $\{\tilde{\mathbf{X}}_i, ~\mathbf{Y}_{i}\}_{i=1}^{T}$. The counterpart optimization problem is the same as Eq. (\ref{eq2}), except that the sum of subscript $i$ is from 1 to $T$.

Notice that the solution to problem in Eq. (\ref{eq2}) is determined by $\mathbf{A}_{[T]}$ and $\mathbf{B}_{[T]}$ defined in Eq. (\ref{eq5}) and Eq. (\ref{eq4}). This evokes us to get the following updating rule.
{
\begin{equation}
\label{eq7}
\begin{split}
\mathbf{A}_{[T+1]} =& \mathbf{A}_{[T]}+
\left[\begin{array}{*{20}{c}}
(1+\lambda) \tilde{\mathbf{X}}_{T+1}^\top \tilde{\mathbf{X}}_{T+1} &  -\lambda \tilde{\mathbf{X}}_{T+1}^\top \mathbf{X}_{T+1}^{\textrm{(s)}} \\
-\lambda (\mathbf{X}_{T+1}^{\textrm{(s)}})^\top \tilde{\mathbf{X}}_{T+1}  & (1+\lambda)(\mathbf{X}_{T+1}^{\textrm{(s)}})^\top \mathbf{X}_{T+1}^{\textrm{(s)}}
\end{array}
\right],\\
\mathbf{B}_{[T+1]} =&  \mathbf{B}_{[T]}+
\left[ {\begin{array}{*{20}{c}}
\tilde{\mathbf{X}}_{T+1}^\top \\
(\mathbf{X}_{T+1}^{\textrm{(s)}})^\top \\
\end{array}} \right]\mathbf{Y}_{T+1},
\end{split}
\end{equation}}
{
\begin{equation}
\label{eq8}
\begin{split}
\mathbf{A}_{1} =
\left[\begin{array}{*{20}{c}}
(1+\lambda) \tilde{\mathbf{X}}_{1}^\top \tilde{\mathbf{X}}_{1}+\rho \mathbf{I} &  -\lambda \tilde{\mathbf{X}}_{1}^\top \mathbf{X}_{1}^{\textrm{(s)}} \\
-\lambda (\mathbf{X}_{1}^{\textrm{(s)}})^\top \tilde{\mathbf{X}}_{1}   & (1+\lambda)(\mathbf{X}_{1}^{\textrm{(s)}})^\top \mathbf{X}_{1}^{\textrm{(s)}}+\rho \mathbf{I}
\end{array}
\right],
\mathbf{B}_{1} =
\left[ {\begin{array}{*{20}{c}}
\tilde{\mathbf{X}}_{1}^\top \\
(\mathbf{X}_{1}^{\textrm{(s)}})^\top \\
\end{array}} \right]\mathbf{Y}_{1}.
\end{split}
\end{equation}}

According to this result, in time $T+1$, we only need to update $\mathbf{A}_{[T+1]}$ and $\mathbf{B}_{[T+1]}$ by adding the matrices calculated based on the data in batch $T+1$. In other words, we just need to store the matrices $\mathbf{A}_{[T]}$ and $\mathbf{B}_{[T]}$ and update them based on Eq.(\ref{eq7}) in each iteration.

This approach has the following advantages: (1) It just needs to store two matrices with size $(d^{\textrm{(v)}}+2d^{\textrm{(s)}})\times (d^{\textrm{(v)}}+2d^{\textrm{(s)}})$ and $(d^{\textrm{(v)}}+2d^{\textrm{(s)}})\times c$. When $d^{\textrm{(v)}}+2d^{\textrm{(s)}} < \sum_{i=1}^{T_1} n_{i}$, it needs less space than storing the whole data. Through this way, we can get the optimal solution to Eq. (\ref{eq2}) by scanning the total data only once. (2) We only make matrix multiplication in updating $\mathbf{A}_{[T]}$ and $\mathbf{B}_{[T]}$, the computational cost is small. In solving Eq. (\ref{eq4}), the most time-consuming step is computing the inverse of a matrix with size $(d^{\textrm{(v)}}+2d^{\textrm{(s)}})\times (d^{\textrm{(v)}}+2d^{\textrm{(s)}})$. Thus, this method is very efficient with large data number and small data dimensionality.

Compared with the data size, when the number of features, i.e., $d^{\textrm{(v)}}+d^{\textrm{(s)}}$, is rather large, it is unwise to compute the inverse of $\mathbf{A}_{[T_1]}$ directly. In this case, we propose another one-pass learning approach in solving the optimization problem in Eq. (\ref{eq4}).

Define $\mathbf{A}_{[0]}=\rho \mathbf{I}$ and $\mathbf{B}_{[0]}= \mathbf{0}$, the updating rule shown in Eq. (\ref{eq7}) can be initialized from $T=0$. Note that $\mathbf{A}_{[0]}^{-1} = (1/\rho) \mathbf{I}$. If we can replace the updating rule of $\mathbf{A}_{[T+1]}$ in Eq. (\ref{eq7}) by the updating rule of $\mathbf{A}_{[T+1]}^{-1}$ with less computational cost in each iteration, the computational burden in calculating $\mathbf{A}_{[T_1]}^{-1}$ will release. Notice that, the added matrix in updating $\mathbf{A}_{[T+1]}$ is not a full rank matrix if $n_{T+1}$ is smaller than $\min\{d^{\textrm{(v)}},d^{\textrm{(s)}}\}$. We will use this property to compute the inverse with low cost.

\begin{lemma}
\label{lemma2}
The updating rule of $\mathbf{A}_{[T+1]}^{-1}$ is
{
\begin{equation}
\label{eq10}
\begin{split}
\mathbf{A}_{[T+1]}^{-1} =
\mathbf{A}_{[T]}^{-1}-\mathbf{A}_{[T]}^{-1}\mathbf{U}_{T+1}\left(\mathbf{I}+\mathbf{U}_{T+1}^\top\mathbf{A}_{[T]}^{-1}\mathbf{U}_{T+1}\right)^{-1} \mathbf{U}_{T+1}^\top \mathbf{A}_{[T]}^{-1}
\end{split}
\end{equation}}
where $\mathbf{U}_{T+1} = [\mathbf{U}_{T+1,1}, \mathbf{U}_{T+1,2}, \mathbf{U}_{T+1,3}]$ and
{
\begin{equation*}
\label{eq9}
\begin{split}
\mathbf{U}_{T+1,1}=\left[ {\begin{array}{*{20}{c}}
\tilde{\mathbf{X}}_{T+1}^\top \\
\mathbf{0}
\end{array}} \right],
\mathbf{U}_{T+1,2}=\left[ {\begin{array}{*{20}{c}}
\mathbf{0}\\
(\mathbf{X}_{T+1}^{\textrm{(s)}})^\top
\end{array}} \right],
\mathbf{U}_{T+1,3}=\left[ {\begin{array}{*{20}{c}}
\sqrt{\lambda} \tilde{\mathbf{X}}_{T+1}^\top\\
-\sqrt{\lambda}(\mathbf{X}_{T+1}^{\textrm{(s)}})^\top
\end{array}} \right].
\end{split}
\end{equation*}}
\end{lemma}

\begin{proof}
Note that, the updating rule of $\mathbf{A}_{[T+1]}$ is shown in Eq. (\ref{eq7}). We now decompose the adding part as
{
\begin{equation}
\label{eq7-ss}
\begin{split}
& \left[\begin{array}{*{20}{c}}
(1+\lambda) \tilde{\mathbf{X}}_{T+1}^\top \tilde{\mathbf{X}}_{T+1}  &  -\lambda \tilde{\mathbf{X}}_{T+1}^\top \mathbf{X}_{T+1}^{\textrm{(s)}} \\
-\lambda (\mathbf{X}_{T+1}^{\textrm{(s)}})^\top \tilde{\mathbf{X}}_{T+1}  & (1+\lambda)(\mathbf{X}_{T+1}^{\textrm{(s)}})^\top \mathbf{X}_{T+1}^{\textrm{(s)}}
\end{array}
\right] =
\left[\begin{array}{*{20}{c}}
\tilde{\mathbf{X}}_{T+1}^\top \tilde{\mathbf{X}}_{T+1} &  \mathbf{0} \\
\mathbf{0}  & \mathbf{0}
\end{array}
\right] \\
+& \left[\begin{array}{*{20}{c}}
\mathbf{0}  & \mathbf{0} \\
\mathbf{0}  & (\mathbf{X}_{T+1}^{\textrm{(s)}})^\top \mathbf{X}_{T+1}^{\textrm{(s)}}
\end{array}
\right]+
\left[\begin{array}{*{20}{c}}
 \lambda \tilde{\mathbf{X}}_{T+1}^\top \tilde{\mathbf{X}}_{T+1} &  -\lambda \tilde{\mathbf{X}}_{T+1}^\top \mathbf{X}_{T+1}^{\textrm{(s)}} \\
-\lambda (\mathbf{X}_{T+1}^{\textrm{(s)}})^\top \tilde{\mathbf{X}}_{T+1}   & \lambda (\mathbf{X}_{T+1}^{\textrm{(s)}})^\top \mathbf{X}_{T+1}^{\textrm{(s)}}
\end{array}
\right]
\end{split}
\end{equation}
}

Denote $\mathbf{U}_{T+1,1}$, $\mathbf{U}_{T+1,2}$ and $\mathbf{U}_{T+1,3}$ as shown in Proposition 2, we have
{
\begin{equation}
\label{eq9-s}
\begin{split}
& \left[\begin{array}{*{20}{c}}
(1+\lambda) \tilde{\mathbf{X}}_{T+1}^\top \tilde{\mathbf{X}}_{T+1}  &  -\lambda \tilde{\mathbf{X}}_{T+1}^\top \mathbf{X}_{T+1}^{\textrm{(s)}} \\
-\lambda (\mathbf{X}_{T+1}^{\textrm{(s)}})^\top \tilde{\mathbf{X}}_{T+1}  & (1+\lambda)(\mathbf{X}_{T+1}^{\textrm{(s)}})^\top \mathbf{X}_{T+1}^{\textrm{(s)}}
\end{array}
\right] \\
= &\mathbf{U}_{T+1,1} \mathbf{U}_{T+1,1}^\top + \mathbf{U}_{T+1,2} \mathbf{U}_{T+1,2}^\top+\mathbf{U}_{T+1,3} \mathbf{U}_{T+1,3}^\top \\
= & \mathbf{U}_{T+1} \mathbf{U}_{T+1}^\top,
\end{split}
\end{equation}
}
with $\mathbf{U}_{T+1} = [\mathbf{U}_{T+1,1}, \mathbf{U}_{T+1,2}, \mathbf{U}_{T+1,3}]$.

Using the Woodbury equation \cite{Higham:2002}, we have the results as follows.
{
\begin{equation}
\label{eq10-s}
\begin{split}
\mathbf{A}_{[T+1]}^{-1} = (\mathbf{A}_{[T]}+ \mathbf{U}_{T+1} \mathbf{U}_{T+1}^\top )^{-1} =
\mathbf{A}_{[T]}^{-1}-\mathbf{A}_{[T]}^{-1}\mathbf{U}_{T+1}\left(\mathbf{I}+\mathbf{U}_{T+1}^\top \mathbf{A}_{[T]}^{-1}\mathbf{U}_{T+1}\right)^{-1} \mathbf{U}_{T+1}^\top \mathbf{A}_{[T]}^{-1}
\end{split}
\end{equation}}

It is the results shown in Proposition 2.
\end{proof}

Similarly, it is also the one-pass way in updating $\mathbf{A}_{[T+1]}^{-1}$ and we need to access the whole data only once. In each iteration shown in Eq. (\ref{eq10}), the most computational step is calculating the inverse of a matrix with size $3n_{T+1}\times 3n_{T+1}$. If the batch size $n_{T+1}$ is small, its computational cost is limited and we can compute $\mathbf{A}_{[T_1]}^{-1}$ in a quick way. Especially, if $n_i=1$ for $i\in [T_1]$ as in traditional online learning, we only need to compute the inverse of a $3\times 3$ matrix.

Besides, a byproduct of this kind of iteration is that we can get the optimal solution at any time $T$, since we have derived $\mathbf{A}_{[T]}^{-1}$ directly. If we use the iteration method shown in Eq. (\ref{eq7}), we need to calculate $\mathbf{A}_{[T]}^{-1}$ at each time $T$. When this requirement is frequent, the computational cost will increase since we need to compute the matrix inverse for each requirement.

\textbf{E-stage:} The optimization problem in Eq. (\ref{eq11}) with concrete forms defined in Eq. (\ref{eq19}) has been widely investigated in previous works and the details are omitted. We use the implementation of LibLinear \cite{REF08a} to solve them and the parameters $w_1$ and $w_2$ are turned by cross validation.

\subsection{Extension}

Note that in Eq. (\ref{eq11}), the interaction between two classifiers is a balance of classification results by turning the weights. We think the more direct way is combining all the features as in stacking \cite{stacking} and training a unified classifier on the stacked representations. This evokes the following formulation.
{
\begin{equation}
\label{eq12}
\begin{split}
\min_{\bar{h}^{\textrm{(s)}}, \bar{h}, w_1, w_2} \ell \left( \sqrt{w_1} \bar{h}^{\textrm{(s)}}(\mathbf{Z}_{T_1+1}^{\textrm{(s)}})+
\sqrt{w_2} \bar{h}(\bar{\mathbf{Z}}_{T_1+1}), \mathbf{Y}_{T_1+1} \right)~,
\end{split}
\end{equation}}
where $\ell$ is a general loss on the joint representations. Here, we use the square root of balance parameters to guarantee their convexity and avoid the trivial solution. They can be learned automatically without extra hyper-parameters.

Taking regression with kernels as an example, we have the following concrete formulation.
{
\begin{equation}
\label{eq13}
\begin{split}
\min_{\mathbf{V}^{\textrm{(s)}},\bar{\mathbf{V}},w_1,w_2 }
& \left \| \sqrt{w_1} \langle {\mathbf{V}}^{\textrm{(s)}}, \Phi(\mathbf{Z}_{T_1+1}^{\textrm{(s)}}) \rangle +
\sqrt{w_2} \langle \bar{\mathbf{V}}, \Phi( \bar{\mathbf{Z}}_{T_1+1}) \rangle - \mathbf{Y}_{T_1+1}  \right\|^2\\
+& \gamma (\frac{1}{c}\| \mathbf{V}^{\textrm{(s)}}\|^2+ \frac{1}{c+d^{\textrm{(a)}}}\| \bar{\mathbf{V}} \|^2),~\textrm{with}~w_1+w_2=1,~w_1\geq 0,~w_2\geq 0~,
\end{split}
\end{equation}}
where $\Phi(\cdot)$ is a mapping function and $\gamma$ is the parameter for regularization. The feature numbers of $\mathbf{Z}_{T_1+1}^{\textrm{(s)}}$ and $\bar{\mathbf{Z}}_{T_1+1}$ are $c$, $c+d^{\textrm{(a)}}$, and $c$ is often much smaller than $c+d^{\textrm{(a)}}$. To alleviate the influence caused by the unbalance, each regularizer is divided by the corresponding feature number.

It is not easy to solve the problem in Eq. (\ref{eq13}) directly. After some deductions, it is equal to
{
\begin{equation}
\label{eq14}
\begin{split}
\min_{\mathbf{V}^{\textrm{(s)}},\bar{\mathbf{V}}, w_1, w_2 }
& \left\| \langle {\mathbf{V}}^{\textrm{(s)}}, \Phi(\mathbf{Z}_{T_1+1}^{\textrm{(s)}}) \rangle +
 \langle \bar{\mathbf{V}}, \Phi( \bar{\mathbf{Z}}_{T_1+1}) \rangle - \mathbf{Y}_{T_1+1}  \right\|^2\\
+\gamma (\frac{1}{c \times w_1} &\| \mathbf{V}^{\textrm{(s)}}\|^2+ \frac{1}{(d^{\textrm{(a)}}+c)\times w_2}\| \bar{\mathbf{V}} \|^2),~\textrm{with}~w_1+w_2=1,~w_1\geq 0,~w_2\geq 0.
\end{split}
\end{equation}}

The optimization problem in Eq. (\ref{eq14}) seems very similar to traditional regularized square loss regression. Nevertheless, they are different since the regularization parameter is fixed in traditional method, while it is also an optimization parameter in our setting.

There are two groups of optimization variables, i.e., the coefficients for classification and the balance parameters. It is difficult to optimize them together and we optimize them alternatively.

When $\mathbf{V}^{\textrm{(s)}}$ and $\bar{\mathbf{V}}$ are fixed, the optimal balance parameters $w_1$ and $w_2$ can be computed by the following proposition directly.

\begin{lemma}
\label{lemma3}
When $w_1+w_2=1,~w_1\geq 0,~w_2\geq 0$,
{
\begin{equation}
\label{eq15}
\begin{split}
\min_{w_1,w_2}~(\frac{1}{ w_1 \times c} \| \mathbf{V}^{\textrm{(s)}}\|^2+ \frac{1}{w_2 \times (d^{\textrm{(a)}}+c)} \|\bar{\mathbf{V}}\|^2) =
(\frac{1}{\sqrt{c}} \| \mathbf{V}^{\textrm{(s)}}\|+ \frac{1}{\sqrt{d^{\textrm{(a)}}+c}} \| \bar{\mathbf{V}}\|)^2.
\end{split}
\end{equation}}
The optimal solution is
{
\begin{equation}
\label{eq16}
\begin{split}
w_1^*=\frac{\|\mathbf{V}^{\textrm{(s)}}\|/ \sqrt{c}} {\|\mathbf{V}^{\textrm{(s)}}\|/ \sqrt{c}+ \|\bar{\mathbf{V}}\| / \sqrt{d^{\textrm{(a)}}+c} },~
w_2^*=\frac{\|\bar{\mathbf{V}} \|/ \sqrt{d^{\textrm{(a)}}+c}} {\|\mathbf{V}^{\textrm{(s)}}\|/ \sqrt{c}+\|\bar{\mathbf{V}}\| / \sqrt{d^{\textrm{(a)}}+c} }.
\end{split}
\end{equation}}
\end{lemma}
\begin{proof}
The optimization problem is
\begin{equation}
\label{eq11-s}
\begin{split}
\min_{w_1,w_2}~&(\frac{1}{ w_1 \times c} \| \mathbf{V}^{\textrm{(s)}}\|^2+ \frac{1}{w_2\times (d^{\textrm{(a)}}+c)} \|\bar{\mathbf{V}} \|^2)~, \\
\textrm{s.t.}~&w_1+w_2=1,~w_1\geq 0,~w_2\geq 0.
\end{split}
\end{equation}

Replace $w_2$ with $w_2=1-w_1$ in Eq. (\ref{eq11-s}), take derivative of this objective function with respect to $w_1$ and set it to zero, we have
{
\begin{equation}
\label{eq12-s}
\begin{split}
-\frac{1}{ w_1^2 \times c} \| \mathbf{V}^{\textrm{(s)}}\|^2 + \frac{1}{(1-w_1)^2 \times (d^{\textrm{(a)}}+c)} \| \bar{\mathbf{V}}\|^2)= 0. \\
\end{split}
\end{equation}
}

By solving this problem and using $w_2=1-w_1$, we have
{
\begin{equation}
\label{eq13-s}
\begin{split}
w_1=\frac{\|\mathbf{V}^{\textrm{(s)}}\|/ \sqrt{c}} {\|\mathbf{V}^{\textrm{(s)}}\|/ \sqrt{c}+ \|\bar{\mathbf{V}}\| / \sqrt{d^{\textrm{(a)}}+c} },~
w_2=\frac{\|\bar{\mathbf{V}} \|/ \sqrt{d^{\textrm{(a)}}+c}} {\|\mathbf{V}^{\textrm{(s)}}\|/ \sqrt{c}+\|\bar{\mathbf{V}}\| / \sqrt{d^{\textrm{(a)}}+c} }.
\end{split}
\end{equation}
}

Note that the solution in Eq. (\ref{eq13-s}) satisfies the constraint $w_1\geq 0,~w_2\geq 0$ automatically. Thus, it is the optimal solution and the results in Proposition 3 hold.
\end{proof}

When $w_1$ and $w_2$ are fixed, the optimal solution to the problem in Eq. (\ref{eq14}) can be obtained in a close form as shown in the following proposition.

\begin{lemma}
\label{lemma4}
When $w_1$ and $w_2$ are fixed, the optimal solution to the problem in Eq. (\ref{eq14}) can be derived by solving
{
\begin{equation}
\label{eq17}
\begin{split}
\left[\begin{array}{*{20}{c}}
(\Phi(\mathbf{Z}_{T_1+1}^{\textrm{(s)}}))^\top \Phi(\mathbf{Z}_{T_1+1}^{\textrm{(s)}})+  \frac{\gamma} {c \times w_1}\mathbf{I}
                      & (\Phi(\mathbf{Z}_{T_1+1}^{\textrm{(s)}}))^\top \Phi( \bar{\mathbf{Z}}_{T_1+1}) \\
                        (\Phi(\bar{\mathbf{Z}}_{T_1+1}))^\top \Phi( \mathbf{Z}_{T_1+1}^{\textrm{(s)}})
&(\Phi( \bar{\mathbf{Z}}_{T_1+1} ))^\top \Phi(\bar{\mathbf{Z}}_{T_1+1})+  \frac{\gamma} {(d^{\textrm{(a)}}+c) \times w_2}\mathbf{I}
\end{array}
\right] \mathbf{V} = \mathbf{D},
\end{split}
\end{equation}}
{
\begin{equation}
\label{eq18}
\begin{split}
\mathbf{V}=
\left[ {\begin{array}{*{20}{c}}
\mathbf{V}^{\textrm{(s)}} \\
\bar{\mathbf{V}} \\
\end{array}} \right],~~
\mathbf{D} = \left[ {\begin{array}{*{20}{c}}
(\Phi(\mathbf{Z}_{T_1+1}^{\textrm{(s)}}))^\top \\
(\Phi(\bar{\mathbf{Z}}_{T_1+1}))^\top \\
\end{array}} \right] \mathbf{Y}_{T_1+1}.
\end{split}
\end{equation}}
\end{lemma}

\begin{proof}
When $w_1$ and $w_2$ are fixed, the optimization problem is shown in Eq. (\ref{eq14}). Take the derivative of it with respect respect to $\mathbf{V}^{\textrm{(s)}}$ and $\bar{\mathbf{V}}$, set them to zeros, we have
{
\begin{equation}
\label{eq17-s}
\begin{split}
&\left( (\Phi(\mathbf{Z}_{T_1+1}^{\textrm{(s)}}))^\top \Phi(\mathbf{Z}_{T_1+1}^{\textrm{(s)}})+  \frac{\gamma} {c \times w_1}\mathbf{I}\right) \mathbf{V}^{\textrm{(s)}}+
 (\Phi(\mathbf{Z}_{T_1+1}^{\textrm{(s)}}))^\top \Phi(\bar{\mathbf{Z}}_{T_1+1}) \bar{\mathbf{V}} = (\Phi(\mathbf{Z}_{T_1+1}^{\textrm{(s)}}))^\top \mathbf{Y}_{T_1+1}, \\
&\left( (\Phi(\bar{\mathbf{Z}}_{T_1+1}))^\top \Phi(\bar{\mathbf{Z}}_{T_1+1})+  \frac{\gamma} {(d^{\textrm{(a)}}+c) \times w_2}\mathbf{I}\right) \bar{\mathbf{V}} +
 (\Phi( \bar{\mathbf{Z}}_{T_1+1}))^\top \Phi(\mathbf{Z}_{T_1+1}^{\textrm{(s)}}) \mathbf{V}^{\textrm{(s)}}
 = (\Phi(\bar{\mathbf{Z}}_{T_1+1}))^\top \mathbf{Y}_{T_1+1}.
\end{split}
\end{equation}
}
After making some notations shown in Proposition 4, we can get the results in Proposition 4.
\end{proof}

There is a mapping function $\Phi$ in this formulation. If we do not know its concrete form, the kernel trick \cite{Scholkopf:1999:AKM:299094} can be employed to make prediction for testing data.

When we obtain the final classifier, it can be used on the testing data $\bar{\mathbf{X}}_{T_1+2} = [\mathbf{X}_{T_1+2}^{\textrm{(s)}}, \mathbf{X}_{T_1+2}^{\textrm{(a)}}]$ by first computing the new representations of $\mathbf{X}_{T_1+2}^{\textrm{(s)}}$ as ($\tilde{h}_{*}^{\textrm{(s)}}(\mathbf{X}_{T_1+2}^{\textrm{(s)}})$) and then employing this classifier. If we take the updating rule in Eq. (\ref{eq7}) and train a classifier in E-stage by optimizing Eq. (\ref{eq13}), the procedure of OPID is listed in Algorithm \ref{alg1:opid}.

\begin{algorithm}[!t]
\caption{OPID}
\label{alg1:opid}
{
\begin{algorithmic}
\STATE \textbf{Input:} Training data and label $\{\tilde{\mathbf{X}}_i, \mathbf{Y}_{i}\}_{i=1}^{T_1}$, $\{\bar{\mathbf{X}}_{T_1+1}, \mathbf{Y}_{T_1+1}\}$, testing data $\bar{\mathbf{X}}_{T_1+2}$, parameters $\lambda$, $\rho$, $\gamma$.
\STATE \textbf{Output:} The optimal $w_1$, $w_2$, $\mathbf{V}^{\textrm{(s)}}$ and $\bar{\mathbf{V}}$.  \\
\STATE \textbf{Training:} \\
\STATE 1: Update $\mathbf{A}_{[T]}$ and $\mathbf{B}_{[T]}$ using Eq. (\ref{eq7}) in the one-pass way. \\
\STATE 2: Compute the optimal $\mathbf{W}_{*}^{\textrm{(s)}}$ by solving Eq. (\ref{eq4}).\\
\STATE 3: Compute the new representation of $\mathbf{X}_{T_1+1}^{\textrm{(s)}}$ by $\mathbf{Z}_{T_1+1}^{\textrm{(s)}} = \mathbf{X}_{T_1+1}^{\textrm{(s)}} \mathbf{W}_{*}^{\textrm{(s)}}$.\\
\STATE 4: Initialize $w_1$=$w_2$=1/2\\
\STATE  \textbf{Repeat} \\
\STATE 5: Update $\mathbf{V}^{\textrm{(s)}}$ and $\bar{\mathbf{V}}$ by solving the problem in Eq. (\ref{eq17}). \\
\STATE 6: Update $w_1$ and $w_2$ by Eq. (\ref{eq16}). \\
\STATE \textbf{Until converges}
\STATE \textbf{Testing}
\STATE 7: Compute the new representation of $\mathbf{X}_{T_1+2}^{\textrm{(s)}}$ by $\mathbf{X}_{T_1+2}^{\textrm{(s)}} \mathbf{W}_{*}^{\textrm{(s)}}$. \\
\STATE 8: Compute the predicted label matrix of $\mathbf{X}_{T_1+2}^{\textrm{(s)}}$ by using the optimal $w_1$, $w_2$, $\mathbf{V}^{\textrm{(s)}}$, $\bar{\mathbf{V}}$ and kernel trick.
\end{algorithmic}}
\end{algorithm}

\section{Experimental Results}

There are two implementations (Eq. (\ref{eq11}) and Eq. (\ref{eq13})) of our algorithm. We name the OPID method with ensemble (Eq. (\ref{eq11})) as OPIDe and the implementation in Eq. (\ref{eq13}) is still named as OPID. For simplicity, in Eq. (\ref{eq13}), we take $\Phi(\mathbf{x})=\mathbf{x}$ and thus the classifiers are linear. For fairness, the classifiers in Eq. (\ref{eq11}) are also linear. We implement it using the LibLinear toolbox \cite{REF08a} with L$_2$-regularized logistic regression and the parameters are tuned by five-fold cross validation. The multi-class problem is tackled by one-vs-rest strategy.

We will compare OPID with other related methods. To show whether the classifier trained in C-stage is helpful, we also train a linear classifier with L$_2$-regularized logistic regression on the whole training data in E-stage, i.e., $\{ \bar{\mathbf{X}}_{T_1+1} \triangleq [\mathbf{X}_{T_1+1}^{\textrm{(s)}}, \mathbf{X}_{T_1+1}^{\textrm{(a)}}],~\mathbf{Y}_{T_1+1}\}$. For simplicity, this method is notated by SVM. Besides, to show the effectiveness of expanding in E-stage, we also train the same kind of linear classifiers on data with survived feature $\{ \mathbf{X}_{T_1+1}^{\textrm{(s)}}, \mathbf{Y}_{T_1+1}\}$ and augmented feature $\{\mathbf{X}_{T_1+1}^{\textrm{(a)}}, \mathbf{Y}_{T_1+1}\}$ respectively. They are named as SVM$^{\textrm{(s)}}$ and SVM$^{\textrm{(a)}}$.

We have evaluated our approach on eight different kinds of data sets. They are three digit data sets: Mnist\footnote{http://yann.lecun.com/exdb/mnist/}, Gisette\footnote{http://clopinet.com/isabelle/Projects/NIPS2003/\#challenge} and
USPS\footnote{http://yann.lecun.com/exdb/mnist/},
three DNA data sets:
DNA\footnote{http://archive.ics.uci.edu/ml/machine-learning-databases/statlog/}, Splice\footnote{http://www.cs.toronto.edu/~delve/data/splice/desc.html} and
Protein\footnote{http://archive.ics.uci.edu/ml/machine-learning-databases/statlog/},
the Vehicle data: SensIT Vehicle\footnote{ http://www.ecs.umass.edu/~mduarte/Software.html}, and
the image data set: Satimage\footnote{http://www.cs.toronto.edu/~delve/data/splice/desc.html}.
The digit data sets are employed as toy examples and the rest data sets are all collected in a sequential way. For example, the SensIT Vehicle data are collected from a sensor networks. Due to the life variance of different kinds of sensors, it is a typical feature and instance evolution system.

For simplicity, we assume that (1) $n_1=n_2=,\cdots,n_{T_1}=n_{T_1+1}=n_{T_1+2}$. In C-stage, the number of training points in each category is equal, whereas in E-stage, we randomly split $n_{T_1+1}+n_{T_1+2}$ examples into two equal parts and assign them as training and testing samples respectively. (2) In C-stage, we fix the total number of examples and vary the points number in each batch. Thus, companied with this, the number of training and testing examples also changes in E-stage. (3) We select the first $(\sum_{i=1}^{T_1} n_i)/c$ samples in each category as the training data in C-stage. (4) We assign the first $d^{\textrm{(v)}}$ features as vanished features, the next $d^{\textrm{(s)}}$ features as survived features and the rest as augmented features. Without specification, in our experiments, the first quarter and the last quarter are vanished features and augmented features, except for DNA and Splice, since their original dimensionality is low. Experimental setting details are shown in Table \ref{table1}.

\begin{table}[ht]
  \caption{{ The details about experimental setting.}}
  \label{table1}
  \vskip 0.1in
  \centering
  {
  \begin{tabular}{cccccccc}
    \hline
    Data & c & $\sum_{i=1}^{T_1} n_i$ & $n_i$ & $d^{\textrm{(v)}}$ & $d^{\textrm{(s)}}$ & $d^{\textrm{(a)}}$ & $n_{T_1+1}=n_{T_1+2}$\\
    \hline
    Mnist0vs5    & 2 & 3200 & 40, 80, 160, 320  & 114 & 228 & 113 & 40, 80, 160, 320 \\
    Mnist0vs3vs5 & 3 & 4800 & 60, 120, 240, 480 & 123 & 245 & 121 & 60, 120, 240, 480 \\
    DNA          & 3 & 1200 & 60, 120, 240, 300 & 50  & 80  & 50  & 60, 120, 240, 300 \\
    Splice       & 2 & 2240 & 40, 80,  160, 320 & 10  & 40  & 10  & 40, 80,  160, 320 \\
    SensIT Vehicle & 3 & 48000& 60, 120, 240, 480 & 25 & 50 & 25 & 60, 120, 240, 480 \\
    Gisette        & 2 & 6000 & 40, 100, 200, 300 & 1239 & 2478 & 1238 & 40, 100, 200, 300\\
    USPS0vs5    & 2 & 960  & 20, 40, 60,  80 & 64 & 128 & 64 & 20, 40, 60,  80\\
    USPS0vs3vs5 & 3 & 1440 & 30, 60, 90, 120 & 64 & 128 & 64 & 30, 60, 90, 120 \\
    Protein     & 3 & 4500 & 60, 150, 300, 450 & 70 & 200 & 86 & 60, 150, 300, 450 \\
    Satimage    & 3 & 1080 & 30, 60,  90, 120  & 10 & 18  & 8  & 30, 60, 90, 120\\
    \hline
  \end{tabular}}
  \vskip -0.1in
\end{table}

There are totally two groups of experiments. In the first group, we would like to report the classification accuracy comparison on the testing data in E-stage. In the second group, we will focus on the performance variation caused by the number of survived features.

\subsection{Classification Accuracy Comparison}

To show the results both intuitively and qualitatively, we report results on the first six data in the form of table and the rest with figures. The results of different methods on different data sets are presented in Table \ref{table_accuracy} and Fig. \ref{acc_comparison-s} respectively.

\begin{table*}[!t]
\renewcommand{\arraystretch}{1.1}
\caption{ The testing accuracies (mean$\pm$std.) of the compared methods on 6 data sets with different number of training and testing examples. '$\bullet/${\tiny $\odot$}$/\circ$' denote respectively that OPID(or OPIDe) is significantly better/tied/worse than the compared method by the $t$-test\cite{t-test} with confidence level 0.05. '-' means that the result is unavailable. From the third column to the fifth column, two symbols are the comparisons to IPODe and IPOD respectively. In the sixth column, the symbols are the comparisons between OPIDe and OPID. The highest mean accuracy is also boldfaced.}
\vskip  0.1in
\label{table_accuracy}
\centering
{\scriptsize
\begin{tabular}{c|| c| c c c  |c c }
\hline
  Data set &$n_i$ & $\textrm{SVM}$ & $\textrm{SVM}^{\textrm{(s)}} $ & $\textrm{SVM}^{\textrm{(a)}}$ & OPIDe & OPID \\
\hline
  Mnist& 40&.9485(.0309)$\bullet\bullet$&    .9455(.0334)$\bullet\bullet$&    .7105(.0569)$\bullet\bullet$&    .9700(.0286)$\bullet$&    \textbf{.9840}(.0158) \\
  0vs5 & 80&.9630(.0196)$\bullet\bullet$&    .9633(.0196)$\bullet\bullet$&    .6403(.0434)$\bullet\bullet$&    .9868(.0088){\tiny $\odot$}&    \textbf{.9888}(.0099) \\
       &160&.9571(.0137)$\bullet\bullet$&    .9528(.0136)$\bullet\bullet$&    .6691(.0271)$\bullet\bullet$&    .9794(.0097){\tiny $\odot$}&    \textbf{.9875}(.0090) \\
       &320&.9604(.0104)$\bullet\bullet$&    .9585(.0102)$\bullet\bullet$&    .6508(.0220)$\bullet\bullet$&    \textbf{.9738}(.0058){\tiny $\odot$}&    .9721(.0066) \\
\hline
  Mnist  &60& .9043(.0424)$\bullet\bullet$&    .9093(.0399)$\bullet\bullet$&    .5137(.0534)$\bullet\bullet$&    .9280(.0339)$\bullet$&     \textbf{.9403}(.0265) \\
  0vs3vs5&120&.9233(.0218)$\bullet\bullet$&    .9240(.0243)$\bullet\bullet$&    .4852(.0381)$\bullet\bullet$&    .9458(.0178){\tiny $\odot$}&    \textbf{.9497}(.0130) \\
         &240&.9125(.0165)$\bullet\bullet$&    .9133(.0134)$\bullet\bullet$&    .4653(.0266)$\bullet\bullet$&    .9345(.0141){\tiny $\odot$}&    \textbf{.9348}(.0135) \\
         &480&.9265(.0110)$\bullet\bullet$&    .9232(.0101)$\bullet\bullet$&    .4843(.0182)$\bullet\bullet$&    .9330(.0086){\tiny $\odot$}&    \textbf{.9337}(.0079) \\
\hline
 DNA&60& .7693(.0495)$\bullet\bullet$&    .8017(.0511)$\bullet\bullet$&    .3063(.0499)$\bullet\bullet$&    .9183(.0246){\tiny $\odot$}&    \textbf{.9253}(.0248) \\
    &120&.8418(.0327)$\bullet\bullet$&    .8630(.0308)$\bullet\bullet$&    .3855(.0428)$\bullet\bullet$&    .9315(.0191){\tiny $\odot$}&    \textbf{.9322}(.0172) \\
    &240&.8732(.0182)$\bullet\bullet$&    .8890(.0203)$\bullet\bullet$&    .3778(.0256)$\bullet\bullet$&    \textbf{.9385}(.0111){\tiny $\odot$}&    .9343(.0132) \\
    &300&.8857(.0174)$\bullet\bullet$&    .8969(.0188)$\bullet\bullet$&    .3957(.0242)$\bullet\bullet$&    \textbf{.9405}(.0120)$\circ$&
    .9348(.0118) \\
\hline
  Splice&40&.6625(.0621)$\bullet\bullet$&    .6755(.0618)$\bullet\bullet$&    .4640(.0513)$\bullet\bullet$&    \textbf{.8150}(.0468){\tiny $\odot$}&     .8025(.0455) \\
        &80&.7512(.0403)$\bullet\bullet$&    .7597(.0446)$\bullet\bullet$&    .4680(.0486)$\bullet\bullet$&    \textbf{.8122}(.0373){\tiny $\odot$}&     .8050(.0353) \\
       &160&.8084(.0258)$\bullet\bullet$&    .8116(.0244)$\bullet\bullet$&    .4829(.0354)$\bullet\bullet$&    \textbf{.8400}(.0203){\tiny $\odot$}&     .8391(.0205) \\
       &320&.8408(.0174)$\bullet\bullet$&    .8387(.0171)$\bullet\bullet$&    .4961(.0232)$\bullet\bullet$&    .8555(.0132){\tiny $\odot$}&     \textbf{.8594}(.0138) \\
\hline
SensIT &  60&.6533(.0536)$\bullet\bullet$&    .6513(.0561)$\bullet\bullet$&    .5533(.0727)$\bullet\bullet$&    .7113(.0443){\tiny $\odot$}&    \textbf{.7147}(.0482) \\
Vehicle& 120&.6418(.0482)$\bullet\bullet$&    .6550(.0439)$\bullet\bullet$&    .5777(.0512)$\bullet\bullet$&    \textbf{.7253}(.0410){\tiny $\odot$}&    .7218(.0365) \\
       & 240&.6816(.0287)$\bullet\bullet$&    .6805(.0293)$\bullet\bullet$&    .5942(.0352)$\bullet\bullet$&    .7237(.0187){\tiny $\odot$}&    \textbf{.7242}(.0183) \\
       & 480&.7088(.0177)$\bullet\bullet$&    .7076(.0164)$\bullet\bullet$&    .6065(.0214)$\bullet\bullet$&    .7206(.0132){\tiny $\odot$}&    \textbf{.7223}(.0139) \\
\hline
Gisette
    &  40&.8795(.0564)$\bullet\bullet$&    .8710(.0566)$\bullet\bullet$&    .8550(.0627)$\bullet\bullet$&    .9600(.0331)$\bullet$&    \textbf{.9710}(.0185) \\
    & 100&.9172(.0276)$\bullet\bullet$&    .9028(.0286)$\bullet\bullet$&    .9116(.0264)$\bullet\bullet$&    .9714(.0128){\tiny $\odot$}&    \textbf{.9756}(.0126) \\
    & 200&.9133(.0155)$\bullet\bullet$&    .9039(.0181)$\bullet\bullet$&    .8998(.0201)$\bullet\bullet$&    \textbf{.9559}(.0095){\tiny $\odot$}&    .9539(.0106) \\
    & 300&.9413(.0106)$\bullet\bullet$&    .9308(.0109)$\bullet\bullet$&    .9231(.0126)$\bullet\bullet$&    \textbf{.9636}(.0069)$\circ$&    .9533(.0093) \\
\hline
\multicolumn{2}{c}{OPIDe: win/tie/loss}& 24/0/0 & 24/0/0 & 24/0/0  & - & 2/19/3 \\
\multicolumn{2}{c}{OPID:  win/tie/loss}& 24/0/0 & 24/0/0 & 24/0/0  & 3/19/2 & -  \\
\hline
\end{tabular}}
\end{table*}

\begin{figure*}[!ht]
\vskip 0in
\centering
\begin{minipage}[h]{7.5 cm}
\centering
\includegraphics[height= 2.5 in]{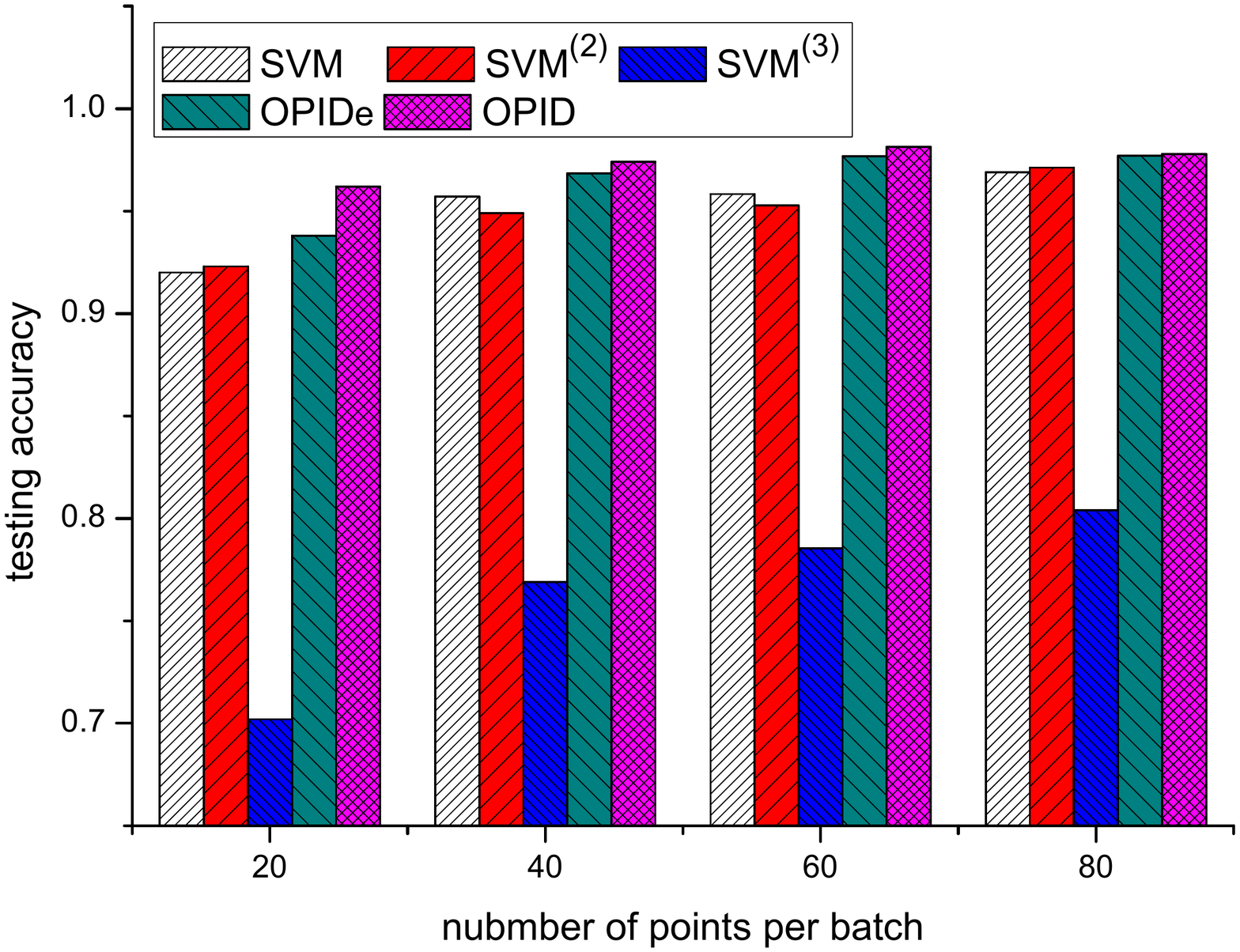}\\
\mbox{\footnotesize~~~~(a) USPS0vs5}
\end{minipage}
\begin{minipage}[h]{7.5 cm}
\centering
\includegraphics[height= 2.5 in]{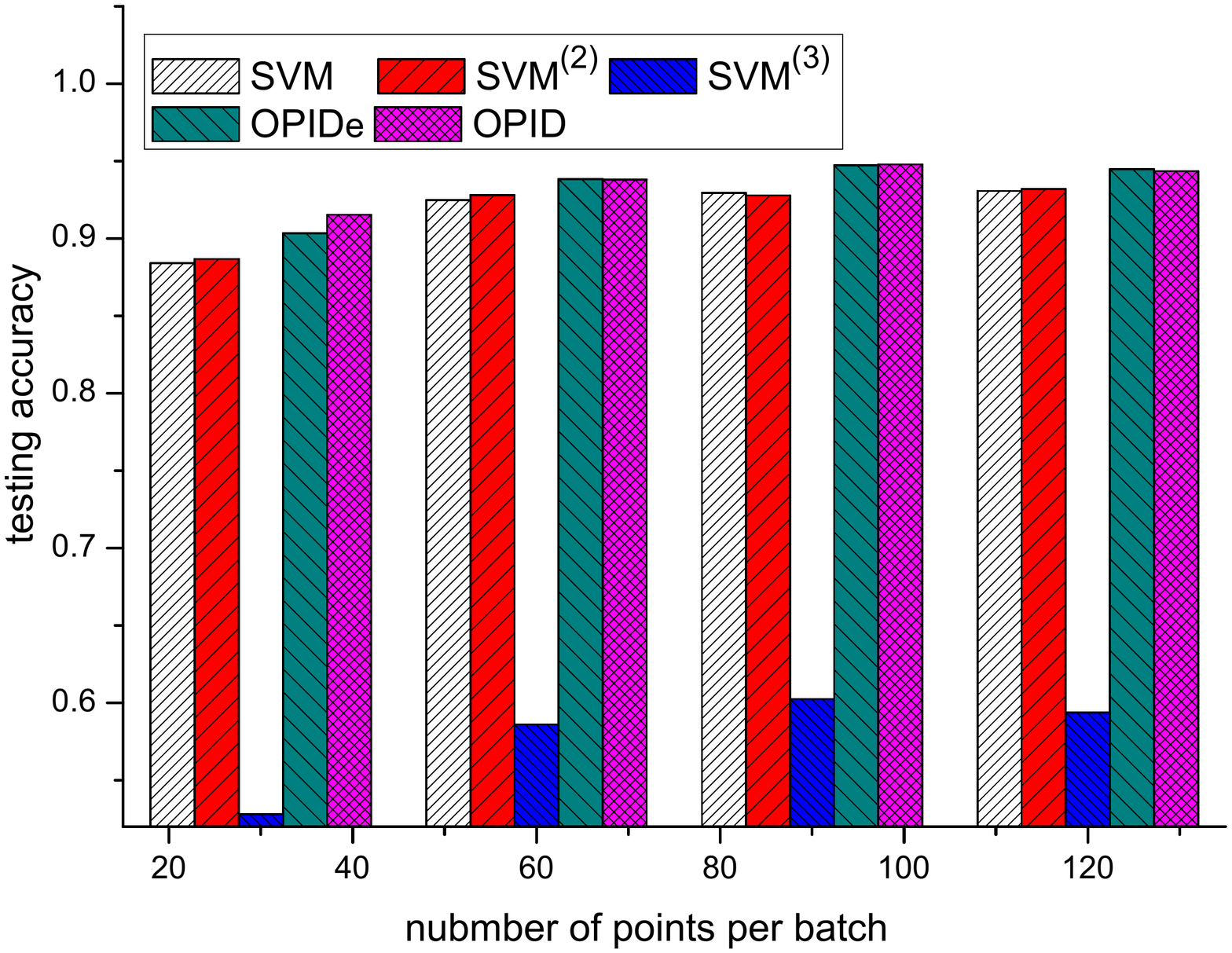}\\
\mbox{\footnotesize~~~~(b) USPS0vs3vs5}
\end{minipage}
\begin{minipage}[h]{7.5 cm}
\centering
\includegraphics[height= 2.5 in]{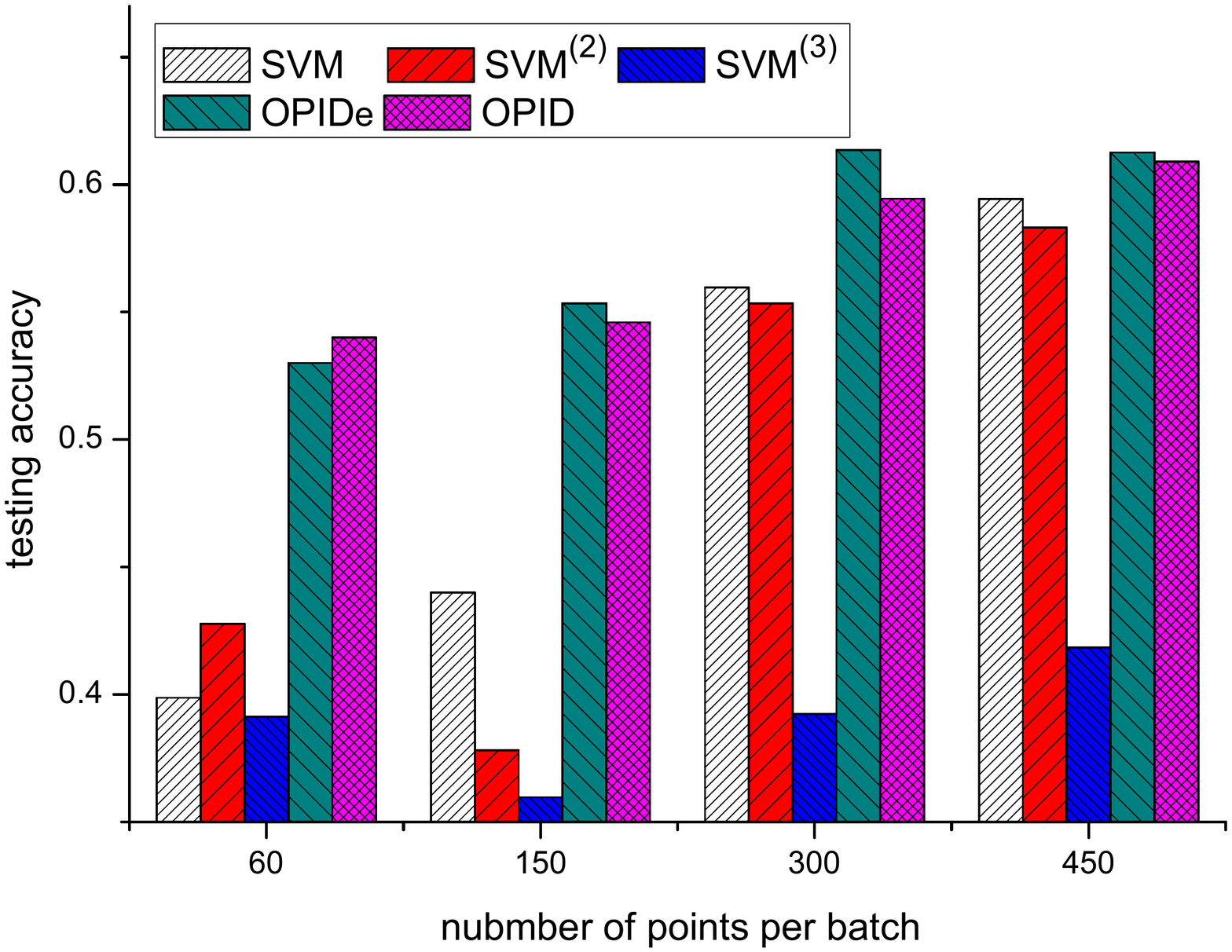}\\
\mbox{\footnotesize~~~~(c) Protein}
\end{minipage}
\begin{minipage}[h]{7.5 cm}
\centering
\includegraphics[height= 2.5 in]{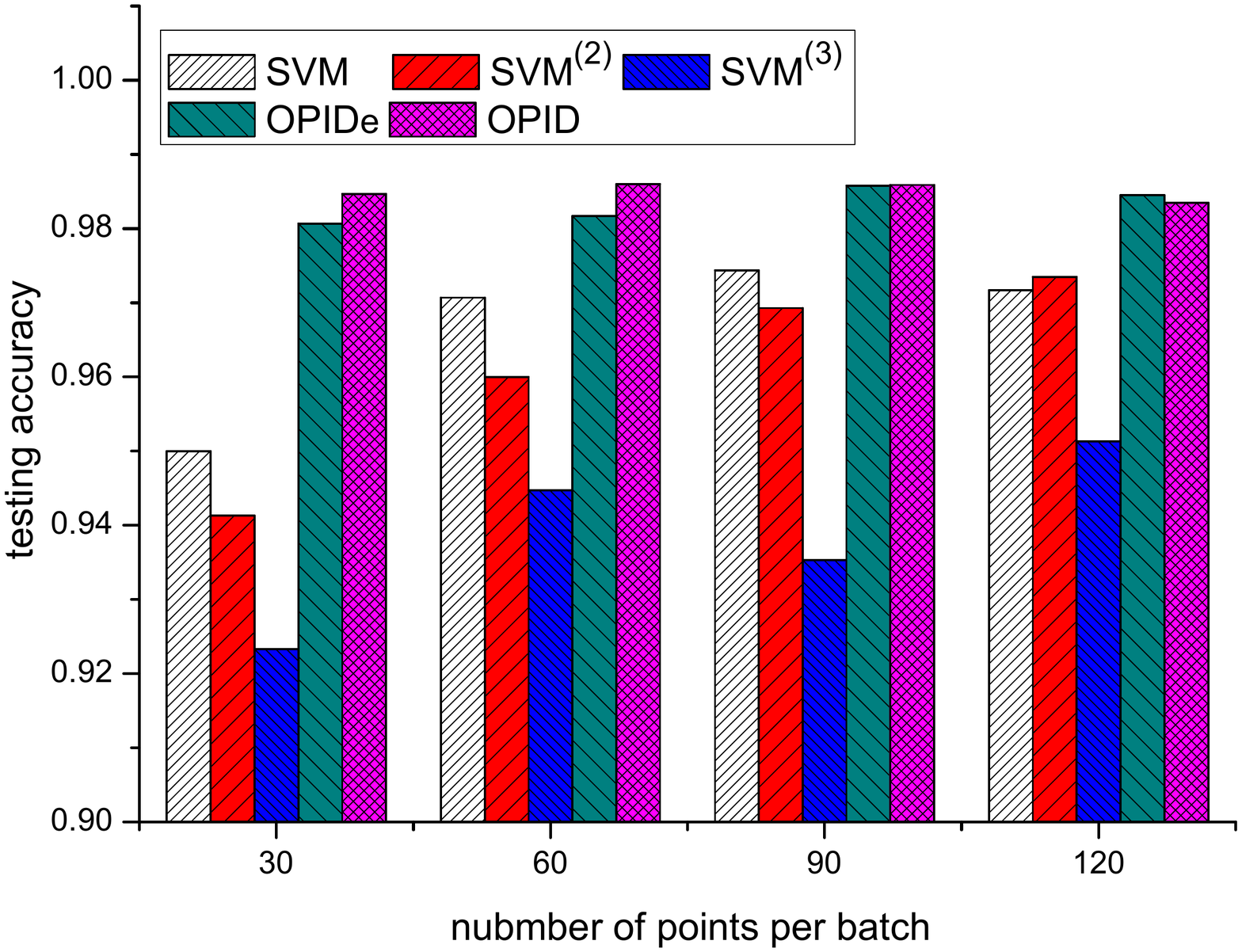}\\
\mbox{\footnotesize~~~~(d) Satimage}
\end{minipage}
\caption{Testing accuracy of different methods on different data sets with different numbers of training and testing instances.}
\label{acc_comparison-s}
\vskip -0.0in
\end{figure*}

 There are several observations from these results.

(1) When we use the classifier trained in C-stage to assist learning in E-stage, the testing performance will increase significantly, especially when the training points in E-stage are rare. This is consistent with intuition since the assistance from C-stage will be weaker with the increase of training points.

(2) Compared with the accuracy of SVM$^{\textrm{(a)}}$, our results have a remarkable improvement. This validates that our methods could, to some extent, inherit the metric from C-stage.

(3) It seems that the improvement of our method with respect to other approaches is much larger in multi-class scenario. The reason may be that the binary classification accuracy is high enough and it is hard to make further improvement.

(4) Compared OPID with OPIDe, it seems that OPID performs slightly better than OPIDe in most data sets. The $t$-test results show that their performances tie in most cases. Nevertheless, their performances are also data dependent.

(5) It seems that our method achieves more significant improvement on biological data sets (DNA, Splice, Protein) than image data sets (Mnist, Gisette and Satimage). It may be caused by the fact that the biological data is more time dependent than the image data and it is more consistent with our settings.

\begin{figure*}[!ht]
\vskip 0in
\centering
\begin{minipage}[h]{8 cm}
\centering
\includegraphics[height= 2.5 in]{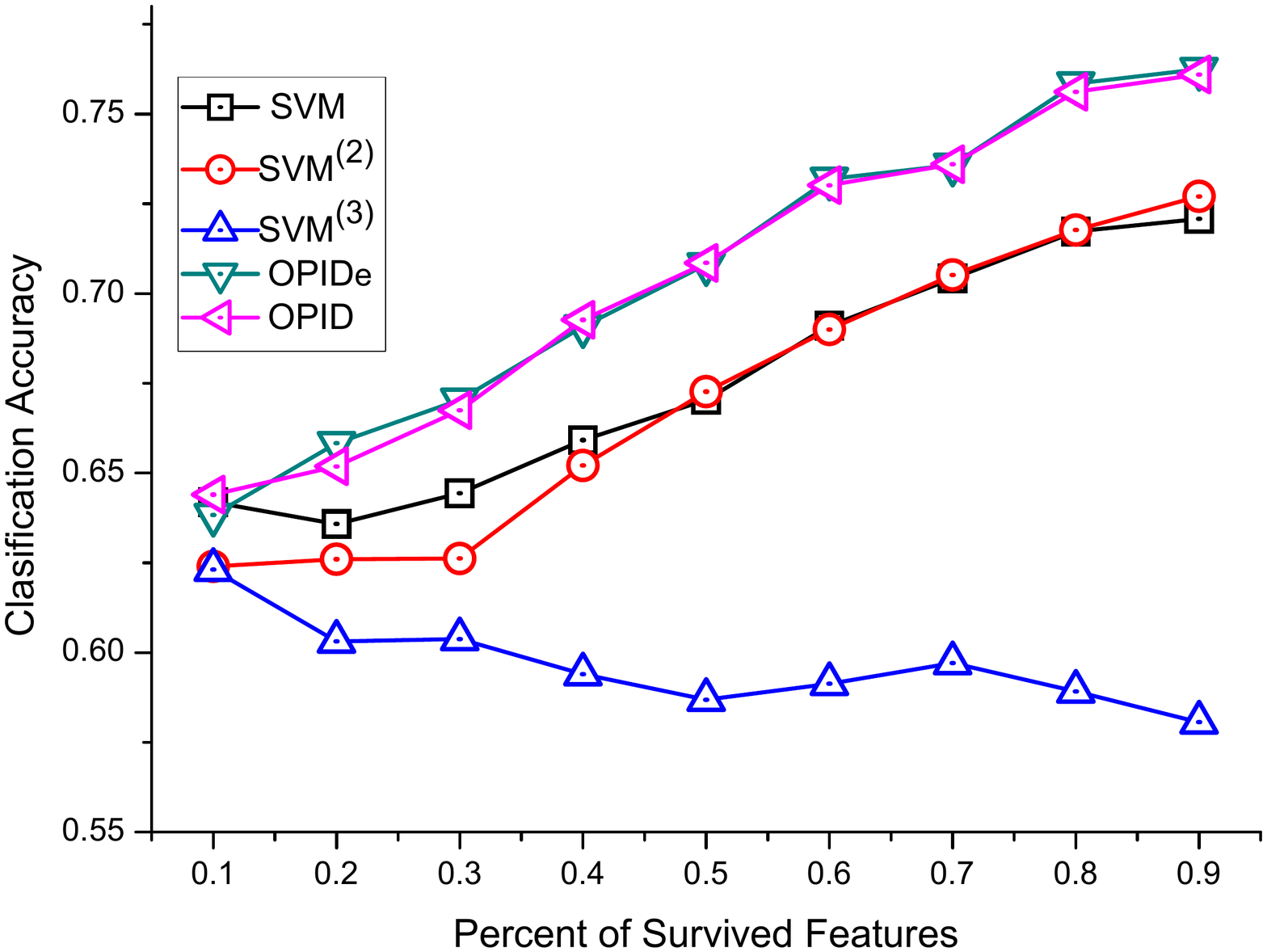}\\
\mbox{\footnotesize (a) SensIT Vehicle}
\end{minipage}
\begin{minipage}[h]{8 cm}
\centering
\includegraphics[height= 2.5 in]{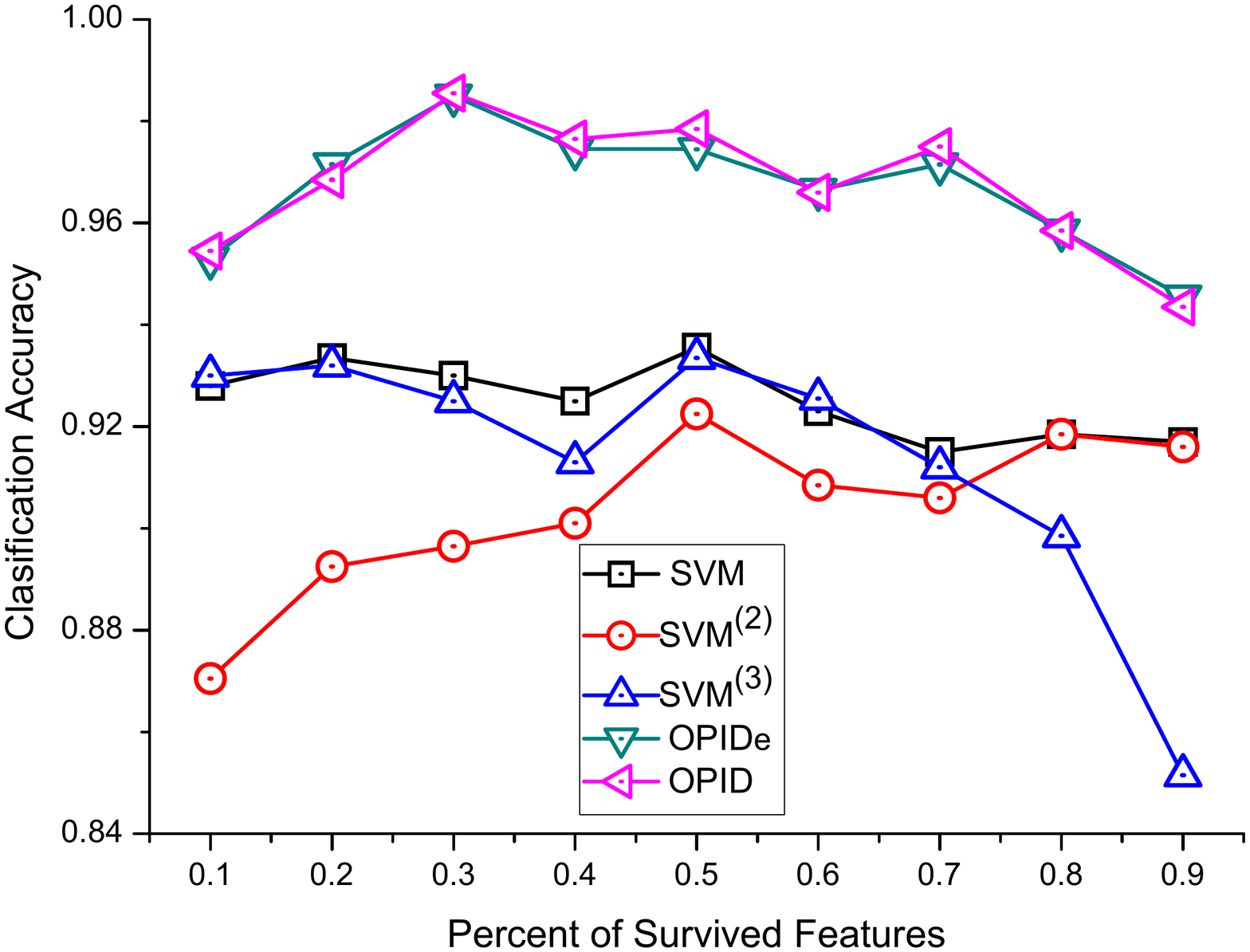}\\
\mbox{\footnotesize (b) Gisette}
\end{minipage}
\caption{Accuracy comparison of different methods with different percents of vanished features, survived features and augmented features.}
\label{dif_par-s}
\end{figure*}

\subsection{The Influence of the Number of Survived Features}

Different from traditional problems, there are three kinds of features in our settings. To illustrate the effectiveness of our methods, we vary the percentages of survived features and compare our methods with other related works. As in previous subsection, we also conduct experiments on SensIT Vehicle and Gisette. Similar to the way of assigning three kinds of features, we select different percentages of features in the middle as the survived feature. The rest are assigned as vanished feature and augmented feature with equal feature number. Comparison results are shown in Fig. (\ref{dif_par-s}).

There are at least two observations from Fig. (\ref{dif_par-s}).

(1) Our proposed methods outperform traditional methods, no matter what the percentage of survived features is. It validates the effectiveness of our method in dealing the problem in this setting.

(2) With the increase number of survived features, OPID and OPIDe achieve higher accuracies on SensIT Vehicle data. Nevertheless, their performances have a little turbulence on Gisette data. The reason may be that compared with SensIT Vehicle data whose total feature number is 100, the feature number of Gisette is 4955. It trends to be suffered from the curse of dimensionality \cite{Donoho00}.

\section{Conclusion}

In this paper, we study the problem of learning with incremental and decremental features, and propose the one-pass learning approach that does not need to keep the whole data for optimization. Our approach is particularly useful when both the data and features are evolving, while robust learning performance are needed, and is well scalable because it only needs to scan each instance once. In this paper we focus on one-shot feature change where the survived and augmented features do not vanish. It will be interesting to extend to multi-shot feature change where the survived and augmented features can vanish later.

\bibliography{opid}\bibliographystyle{alpha}
\end{document}